\documentclass[11pt]{article}
\usepackage[margin=1in]{geometry}
\usepackage{times}

\usepackage{subcaption}
\usepackage{graphicx}
\usepackage{float}
\usepackage{xcolor}
\usepackage{booktabs}

\usepackage{amsmath,amssymb,amsfonts,latexsym,amsthm,mathtools}
\usepackage{enumerate}
\usepackage{bm}
\usepackage{bbm}
\usepackage{dsfont}
\usepackage{comment}
\usepackage{thmtools}

\usepackage[numbers]{natbib}

\usepackage[colorlinks=true,linkcolor=blue]{hyperref}
\hypersetup{
  colorlinks,
  linkcolor={red!60!black},
  citecolor={blue!60!black},
  urlcolor={blue!80!black}
}

\usepackage{cleveref}
\crefname{algorithm}{Algorithm}{Algorithms}

\usepackage{pgfplots}
\pgfplotsset{compat=1.18}
\usetikzlibrary{angles, quotes}

\newtheorem{theorem}{Theorem}

\newtheorem{question}{Question}
  \newtheorem*{question*}{Question}

\newtheorem{corollary}{Corollary}
\newtheorem{definition}{Definition}
\newtheorem{lemma}{Lemma}

\newcommand{\R}{\mathbb{R}}

\newcommand{\x}{{x}}
\newcommand{\Dtrain}{\mathcal{D}}
\newcommand{\Dtest}{\mathcal{D}'}

\newcommand{\Dtrainx}{\mathcal{D}_{\x}}
\newcommand{\Dtestx}{\mathcal{D}_{\x}'}
\newcommand{\DtrainA}{\mathcal{D}_A}
\newcommand{\DtestA}{\mathcal{D}_A'}

\newcommand{\classG}{\mathcal{G}}
\DeclareMathOperator{\VCdim}{VCdim}
\DeclareMathOperator{\sign}{sign}
\DeclareMathOperator{\Span}{span}
\DeclareMathOperator{\Feat}{Feat}
\DeclareMathOperator{\err}{err}
\DeclareMathOperator*{\argmin}{arg\,min}
\newcommand{\opt}{\mathrm{opt}}

\def\Norm#1{\left\| #1 \right\|}
\def\Paren#1{\left( #1 \right)}

\ifdefined\SHOWCOMMENTS
  \newcommand{\jiawei}[1]{{\color{brown}[Jiawei: #1]}}
  \newcommand{\linlin}[1]{{\color{blue}[Lin Lin: #1]}}
\else
  \newcommand{\jiawei}[1]{}
  \newcommand{\linlin}[1]{}
\fi

\begin{document}

\title{Sparsity and Out-of-Distribution Generalization}
\author{
Scott Aaronson\thanks{\texttt{aaronson@cs.utexas.edu}. Supported by Coefficient Giving.} \\ UT Austin \and Lin Lin Lee \thanks{\texttt{llee3@utexas.edu}. Supported by Coefficient Giving.} \\ UT Austin \and Jiawei Li\thanks{\texttt{davidlee@cs.utexas.edu}. Supported by Coefficient Giving.}\\ UT Austin
}
\date{}
\maketitle
\begin{abstract}
    Explaining out-of-distribution generalization has been a central problem in epistemology since Goodman’s ``grue'' puzzle in 1946. Today it’s a central problem in machine learning, including AI alignment.

    Here we propose a principled account of OOD generalization  with three main ingredients. First, the world is always presented to experience not as an amorphous mass, but via distinguished features (for example, visual and auditory channels). Second, Occam’s Razor favors hypotheses that are “sparse,” meaning that they depend on as few features as possible. Third, sparse hypotheses will generalize from a training to a test distribution, provided the two distributions sufficiently overlap on their restrictions to the features that are either actually relevant or hypothesized to be. The two distributions could diverge arbitrarily on other features.

    We prove a simple theorem that formalizes the above intuitions, generalizing the classic sample complexity bound of Blumer et al.~\citep{blumer1989} to an OOD context. We then generalize sparse classifiers to \emph{subspace juntas}, where the ground truth classifier depends solely on a low-dimensional linear subspace of the features.
\end{abstract}

\section{Introduction}
Centuries ago, David Hume asked why we have any rational grounds to expect that the Sun will rise tomorrow---or more generally, why \emph{any} hypothesis formed to explain past data should generalize to unseen data. Today, the spectacular successes of machine learning, and especially the worries about dangers to humanity as AI grows more powerful, motivate revisiting Hume's question from a new lens.

In the 1980s, the new field of computational learning theory had notable success at explaining in-distribution generalization --- that is, the ability of learning algorithms, given enough sample data drawn from a distribution $\mathcal{D}$, to explain further data also drawn from $\mathcal{D}$ (\citep{valiant1984}, \citep{blumer1989}).
Alas, these results fall short of explaining the success of modern deep learning in at least two respects. The first is that modern deep learning is typically ``overparameterized," which is to say: the VC-dimension of the hypothesis class is typically far too large, and the sample size is far too small, for the standard theorems to explain the observed successes. More relevant for us, generalization bounds depend essentially on the assumption that the training distribution and test distribution are the same $\mathcal{D}$ --- i.e., that ``the student is tested only on types of problems that were covered in class."

The theorems are powerless to explain even the most trivial instances of OOD generalization, ones where no one seriously expects that practical ML systems will fail. As an example, suppose a neural network is trained to distinguish photos of cats from photos of dogs, until it does extremely well on the training distribution. But suppose that, in the training images, the top-left pixel always happens to be red. Now the neural network is given a new cat or dog image, where the top-left pixel is yellow. A priori, nothing rules out that the network will suddenly start identifying cats as dogs and dogs as cats. Indeed, if we let $x$ be 0 or 1 depending on whether the image is of a cat or dog, and $y$ be 0 or 1 depending on whether the top-left pixel is red or yellow, nothing rules out that the network has learned $x \oplus y$, when we wanted it to learn $x$. The two hypotheses are equally consistent with the training data, and nearly equally easy to represent in a neural network --- yet they behave in diametrically opposite ways on the test data.

This is simply a modern version of what Nelson Goodman \citep{goodman1955}, in 1946, called the ``New Riddle of Induction," or ``grue" riddle. Define the word \emph{grue} to mean ``green until January 1, 2030, and blue thereafter". Likewise, define \emph{bleen} to mean ``blue until January 1, 2030, and green thereafter". Then Goodman observes that all of our evidence seeming to indicate, for example, that emeralds are green and the sky is blue, could equally well be considered evidence that emeralds are grue and the sky is bleen. A standard response would be that grue and bleen are disfavored by Occam's Razor: why introduce the gratuitous, unmotivated belief that emeralds and the sky will switch color on a specific date, in the absence of any evidence? Goodman responds that, if we imagined a culture where grue and bleen were basic conceptual categories, then \emph{green} and \emph{blue} would seem gratuitous and unmotivated: indeed, one could only express them via complicated phrases like ``grue until January 1, 2030, and bleen thereafter".

Likewise with the image classifier, one could imagine a parameterization of the hypothesis class in which hypotheses like $x \oplus y$ (i.e., cat or dog XOR the color of the top-left pixel) were natural to express, and hypotheses like $x$ (i.e., cat or dog) were complicated. In that case, exposed to images of cats and dogs wherein the top-left pixel always happens to be red, our learner would output an $x \oplus y$ hypothesis. We conclude that, even in this toy scenario, any explanation for the success of OOD generalization must appeal to some principle beyond those considered in 1980s computational learning theory. For even after we fix the training distribution, test distribution, sample data, and hypothesis class, OOD generalization will sometimes succeed and sometimes fail, depending on which hypotheses the learning algorithm favors as defaults.

\subsection{Our Approach}

In our paper, we consider both \textit{sparse} hypotheses, which depend only on a few of the input features, and \textit{subspace juntas}, which generalize this dependency to a low-dimensional subspace of the input space. Sparse hypotheses are one way to formalize the idea of Occam's razor, under the hypothesis that the world is processed as distinguished features. 

We now model the grue puzzle to illustrate how sparsity can help.
Let $\mathcal{X} = \mathbb{R}$, $n = 2$, and $k = 1$. There are $n = 2$ input features: the time $t$ and the ``emeraldness" $e$, which takes the value 1 for an emerald and 0 for a sapphire. Given an input of the form $(t, e)$, the goal is to learn a function $f(t, e)$, which outputs 1 if the object is green or 0 if it's blue. The target function, of course, is $f(t, e) = e$: that is, emeralds are green and sapphires are blue, regardless of the time. The difficulty is that, if we've only seen examples involving times $t < T$, then in principle we could also learn
\begin{equation*}
f(t, e) = \begin{cases} e, & \text{if } t < T, \\ 1-e, & \text{otherwise.} \end{cases}
\end{equation*}
From our perspective, however, the key is that this ``grue" hypothesis is not 1-sparse, depending as it does on both input features. The ``green" hypothesis, $f(t, e) = e$, is preferable because of its dependence on only one feature.

In some instances, however, features may not so clearly represent distinguished concepts. The coordinate basis chosen to represent the data may be arbitrary: the same data may appear different under bases which are rotations of the conventional basis.  Indeed, in modern neural nets, the first layer often applies an arbitrary linear transformation of the input data, which can turn coordinate-sparse dependence into dependence on many coordinates.
Occam's Razor intuitively should choose a notion of simplicity which does not change depending on an arbitrary choice of conventional basis. Subspace juntas provide a basis-robust answer: they formalize the idea that only a few degrees of freedom matter, even if the information is distributed across many coordinates.

\subsection{Our Results}
We give a PAC-style explanation for when out-of-distribution (OOD) generalization should succeed in both the realizable and agnostic settings.

\paragraph{Sparse hypotheses.}
Let $f$ be the ground truth and $h$ be the learned hypothesis.
When $f$ is $k$-sparse (i.e., depending on at most $k$ out of $n$ features) and we restrict our attention to $k$-sparse hypotheses, we show that OOD transfer holds to any test distribution $\Dtest$ that matches (or approximately matches) the training distribution $\Dtrain$ on the features that $f$ and $h$ actually use, even if $\Dtest$ differs arbitrarily on all other features. Concretely, after
\[
m = \tilde{O}\!\left(\frac{d+k\log n}{\epsilon}\right)
\]
training samples, with high probability every $k$-sparse hypothesis consistent with the training samples has error at most $\epsilon$, not only on $\Dtrain$ but on all such distributions $\Dtest$ (\Cref{orig_feat_theorem,thm:orig_sparse_thm}). Here $d$ is an upper bound on the VC-dimension of the underlying hypothesis family when restricted to any $k$ features, and the additional $k\log n$ term is the ``price'' of searching over which $k$ features matter.

\paragraph{Subspace juntas (basis-robust sparsity).}
We also develop a basis-invariant analogue of sparsity via \emph{subspace juntas}.
We prove the analogous OOD guarantee: if both $f$ and $h$ depend only on a subspace $A$, then (approximately) matching
the distribution of the projection onto $A$ between training and test distributions is sufficient for transfer (\Cref{thm:alpha_dist_subspace_juntas}), even if the distributions differ arbitrarily in directions orthogonal to $A$.
Finally, we discuss when these subspace-junta classes have finite VC-dimension, giving conditions (e.g.,
semi-algebraicity) under which one obtains explicit polynomial VC bounds, which have a linear dependence on the dimension $n$.

\paragraph{Agnostic generalization.}
We extend our guarantees for sparse hypotheses and subspace juntas to arbitrary labeled training and test distributions. Let $\opt$ be the minimum training error achievable by the hypothesis class, let $h$ be a hypothesis minimizing empirical training loss, let $\lambda$ be the minimum sum of training and test error, and let $f^*$ achieve this minimum joint error. When the training and test marginals agree on either the relevant features or subspace that $h$ and $f^*$ depend on, we prove an error guarantee of $\opt + \lambda + \epsilon$. We generalize this bound when the marginals only approximately match on the relevant features or subspace.

We note that our results may apply even to distributions $\Dtrain$ and $\Dtest$ which have TV distance $1$ due to the arbitrary allowed behavior outside the relevant features or subspace. We also illustrate the relationship between sparsity and OOD performance on a real-world dataset (\Cref{app:experiments}).

\section{Related Work}

There has been a great deal of recent technical work on OOD generalization---the vast majority of it experimental, focused on predicting when OOD generalization will or will not succeed in practice, and helping it succeed more often \citep{wang2023dg_survey,zhou2022dg_survey,gulrajani2021domainbed,sagawa2019groupdro}. To the best of our knowledge, few papers have sought to identify principles that explain the success of OOD across a wide range of contexts, at a level analogous to the foundational in-distribution generalization principles identified by Valiant \citep{valiant1984} and Blumer et al. \citep{blumer1989} in the 1980s.

One exception is the field of domain adaptation, which began with the work of Ben-David et al. \citep{bendavid2006analysis, bendavid2010theory}. They show general bounds on the test error using the training error and a discrepancy term, which measures how well a hypothesis class can distinguish the training and test distributions. The result of Ben-David et al. shows that OOD generalization succeeds if no hypothesis class can distinguish the training distribution from the test distribution with large bias. Unfortunately, this condition is overly strong; it seems virtually never to be satisfied in the cases we care about in practice. Similarly, while there have been several follow-up works (e.g. \citep{Blitzer07, mmr09,redko2017theoretical,redko2020survey}) which prove relevant bounds and algorithms incorporating the discrepancy distance term, these bounds can prove too conservative to be useful in settings where the distributions differ greatly on irrelevant coordinates. The discrepancy term can be maximal even when OOD generalization is intuitively straightforward (e.g. in our toy example where a single pixel flips between training and test), rendering the bounds vacuous. Impossibility results such as \citep{bendavid2010impossibility} show that we cannot guarantee generalization without some relationship between the training and test distributions. We address this with a different sufficient condition, guaranteeing OOD generalization whenever the train and test distributions approximately overlap on the coordinates or subspace which the learned hypothesis and the ground truth depend on, allowing for arbitrary behavior elsewhere.

Another line of work tackles OOD generalization by formalizing notions of invariance across multiple training environments. For instance, invariant risk minimization \citep{arjovsky2019invariant} proposes learning a data representation such that the same optimal hypothesis can be shared across all training environments. More recent work, such as that of Ye et al. \citep{Ye2021Towards}, has aimed to provide bounds for domain generalization by quantifying how the invariance of learned features on several training environments translates to test environments via an expansion function. Rather than proposing an invariance learning objective or relying on multiple environments, our results give a PAC-style sufficient condition for OOD generalization.

Finally, our sparsity assumptions relate to the classical study of learning a few relevant variables (juntas, see for example \citep{blum1994relevant, mossel04juntas}) and a relevant low dimensional subspace (subspace juntas, see for example \citep{blum1997learning, vempala2013complexity, AV06projection}). While these mainly address finding efficient algorithms for learning these functions under distributional assumptions such as uniformity or Gaussianity, we use sparsity and dependence on a low-dimensional subspace primarily to formalize Occam's Razor style arguments and to cleanly state sufficient conditions for OOD generalization.

\section{Preliminaries}

We introduce some notation and results from PAC learning theory which will be useful.

\paragraph{Notation.} We use $\Dtrain$ and $\Dtest$ for the labeled training and test distributions, and $\Dtrainx$ and $\Dtestx$ for their marginals over the inputs. For $A \subseteq [n]$, we denote the marginal distribution of $x_A$ by $\DtrainA$ (and similarly for $\DtestA$). When $A$ is a linear subspace, $\Dtrain_A$ and $\Dtest_A$ instead denote the distributions of the projection $P_A x$. We write
\[
\err_{\Dtrain}(h) := \Pr_{(x,y) \sim \Dtrain}[h(x) \ne y]
\]
for the classification error of $h$. For a sample $S = ((x_i,y_i))_{i=1}^m$, empirical training error is the fraction of sample labels misclassified, $\frac{1}{m}\sum_{i=1}^m \mathbf{1}\{h(x_i) \ne y_i\}$.

\paragraph{Agnostic benchmarks.} We use the standard source and joint error benchmarks from domain adaptation \citep{bendavid2006analysis,bendavid2012quantity}. For a hypothesis class $\mathcal{F}$, $\opt := \min_{f \in \mathcal{F}} \err_{\Dtrain}(f)$ is the minimum achievable training error, $\lambda := \min_{f \in \mathcal{F}} [\err_{\Dtrain}(f) + \err_{\Dtest}(f)]$ is the minimum sum of training and test error, and $f^* \in \argmin_{f \in \mathcal{F}} [\err_{\Dtrain}(f) + \err_{\Dtest}(f)]$ is a function achieving $\lambda$.

\begin{definition}
Let $\mathcal{X}$ be a sample space, and let $\mathcal{H}$ be a class of functions $h:\mathcal{X}\rightarrow\{0,1\}$. We say that $\mathcal{H}$ shatters a set $X=\{x_{1},...,x_{k}\}\subseteq \mathcal{X}$ if for all $2^{k}$ possible choices for $h(x_{1}),...,h(x_{k})$, there exists an $h\in\mathcal{H}$ consistent with those choices. Then the VC-dimension of $\mathcal{H}$, or $\VCdim(\mathcal{H})$, is the size $k$ of the largest set shattered by $\mathcal{H}$ (or $\VCdim(\mathcal{H})=\infty$ if there is no maximum).
\end{definition}

\begin{lemma}[Sauer's Lemma]\label{lemma:Sauer}
Let $\mathcal{X}$ be a set of size $n$ and let $\mathcal{H}$ be a class of functions $h: \mathcal{X}\rightarrow\{0,1\}$. Then
\[
|\mathcal{H}|\le\sum_{i=0}^{\VCdim(\mathcal{H})}\binom{n}{i}.
\]
\end{lemma}

Using \Cref{lemma:Sauer}, Blumer et al. proved the following famous result.

\begin{theorem}\label{thm:VCdim_PAC} Let $\mathcal{X}$ be a sample space, let $\mathcal{D}$ be a labeled distribution over $\mathcal{X} \times \{0,1\}$, and let $\epsilon$, $\delta>0$. Let $\mathcal{H}$ be a hypothesis class, consisting of hypotheses $h: \mathcal{X}\rightarrow\{0,1\}$ and let $f\in\mathcal{H}$ be the ``ground truth." Then there exists a constant $C$ such that, for all
\[
m\ge\frac{C \VCdim (\mathcal{H})}{\epsilon}\log\frac{1}{\delta\epsilon}
\]
the following holds. With probability at least $1-\delta$ over samples $x_{1},...,x_{m}$ drawn independently from $\Dtrainx$, any hypothesis $h\in\mathcal{H}$ such that $h(x_{i})=f(x_{i})$ for all $i\in[m]$ ``generalizes,'' in the sense that
\[
\Pr_{x\sim\Dtrainx}[h(x)=f(x)]\ge1-\epsilon.
\]
\end{theorem}

\subsection{Shifting between Distributions}
We now take a tiny first step towards OOD generalization by introducing a quantity which measures the worst-case probability amplification for events between two distributions. This allows us to understand how much more likely a rare bad event in one distribution might become in another distribution.

\begin{definition}\label{def:rho}
Given two distributions $\Dtrain$ and $\Dtest$ over the same sample space $\mathcal{X}$ and $\epsilon \in (0,1]$, define
\[
\rho_{\Dtrain,\Dtest}(\epsilon)
:= \inf_{E:\,\Pr_{\Dtest}[E] \ge \epsilon}
\frac{\Pr_{\Dtrain}[E]}{\Pr_{\Dtest}[E]}.
\]
This is the $\epsilon$-weight ratio defined in \citep[Definition 4]{bendavid2012quantity}, with the infimum taken over all events.
\end{definition}

This quantity reflects a worst-case scaling of the probability of an event between the train and the test distributions. In particular, for any event $E$ occurring with probability at least $\epsilon$ under the test distribution, its probability under the training distribution is at least $\rho_{\mathcal{D},\mathcal{D'}}(\epsilon)\Pr_{\mathcal{D'}}[E]$. Intuitively, a lower bound on $\rho_{\Dtrain,\Dtest}(\epsilon)$ means that a bad event $E$ occurring in the test distribution with substantial probability cannot have been almost invisible during training. Note that $0 \le \rho_{\Dtrain,\Dtest}(\epsilon) \le 1$, because taking $E$ to be the entire sample space gives $\Pr_{\Dtrain}[E] = \Pr_{\Dtest}[E] = 1$.

We now prove a generalization of \Cref{thm:VCdim_PAC} that applies even when the training distribution $\mathcal{D}$ and test distribution $\mathcal{D}'$ are different, so long as $\rho_{\Dtrainx,\Dtestx}(\epsilon) > 0$ at the desired error threshold $\epsilon$.

\begin{theorem}\label{thm:orig_alpha}
Let $\mathcal{X}$ be a sample space, let $\Dtrainx, \Dtestx$ be two distributions over $\mathcal{X}$, and let $\epsilon \in (0,1]$ and $\delta > 0$, with $\rho_{\Dtrainx,\Dtestx}(\epsilon) > 0$. Let $\mathcal{H}$ be a hypothesis class, consisting of hypotheses $h: \mathcal{X} \rightarrow \{0,1\}$, and let $f \in \mathcal{H}$ be the ``ground truth." Then there exists a constant $C$ such that, for all
\[
m \ge \frac{C \cdot \VCdim(\mathcal{H})}{\epsilon \rho_{\Dtrainx,\Dtestx}(\epsilon)} \log \frac{1}{\delta \epsilon \rho_{\Dtrainx,\Dtestx}(\epsilon)} ,
\]
the following holds. With probability at least $1-\delta$ over samples $x_{1},...,x_{m}$ drawn independently from $\Dtrainx$, any hypothesis $h \in \mathcal{H}$ such that $h(x_{i}) = f(x_{i})$ for all $i \in [m]$ ``generalizes to $\Dtestx$," in the sense that
\[
\Pr_{x \sim \Dtestx} [h(x) = f(x)] \ge 1 - \epsilon .
\]
\end{theorem}

\begin{proof}
Let $h \in \mathcal{H}$ be any hypothesis such that
\[
\Pr_{x \sim \Dtestx}[h(x) \ne f(x)] \ge \epsilon .
\]
Then by definition of $\rho_{\Dtrainx,\Dtestx}(\epsilon)$, we have
\[
\Pr_{x \sim \Dtrainx}[h(x) \ne f(x)] \ge \rho_{\Dtrainx,\Dtestx}(\epsilon)\Pr_{x \sim \Dtestx}[h(x) \ne f(x)] \ge \epsilon \rho_{\Dtrainx,\Dtestx}(\epsilon).
\]
But this is the only ``input" that the proof of \Cref{thm:VCdim_PAC} needs, and means that we can redo that proof with $\epsilon \rho_{\Dtrainx,\Dtestx}(\epsilon)$ in place of $\epsilon$.
\end{proof}

Unfortunately, \Cref{thm:orig_alpha} does not cover most cases of OOD generalization that we care about in practice, where we really might have $\rho_{\Dtrainx,\Dtestx}(\epsilon) = 0$---that is, where the test distribution $\mathcal{D}'$ might include points that never occurred at all in the training distribution $\mathcal{D}$. This motivates us to consider more structured sample spaces, in which each sample point is divisible into $n$ features. We then consider pairs of distributions $\mathcal{D}, \mathcal{D}'$ that overlap on some features but could be arbitrarily far apart on other features.

\section{Sparse Hypotheses}

Formally, our sample space will now have the form $\mathcal{X}^{n}$, where $\mathcal{X}$ is the feature space and $n$ is the number of features. We consider families of base hypothesis classes of the form $\mathcal{H} = (\mathcal{H}_{k})_{k \ge 1}$, where each $\mathcal{H}_{k}$ consists of functions $h: \mathcal{X}^{k} \rightarrow \{0,1\}$. We then let $\hat{\mathcal{H}}_{k}$, the class of $k$-sparse hypotheses, consist of all functions $\hat{h}: \mathcal{X}^{n} \rightarrow \{0,1\}$ that have the form
\[
\hat{h}(x_{1},...,x_{n}) = h(x_{i(1)},...,x_{i(k)}),
\]
for some $h \in \mathcal{H}_{k}$ and some distinct $i(1),...,i(k) \in [n]$. And we let $\hat{\mathcal{H}} := (\hat{\mathcal{H}}_{k})_{k \ge 1}$. To analyze these classes of sparse hypotheses, we first need a technical lemma. The bound on the VC-dimension of a union of hypothesis classes with bounded VC-dimension is standard (e.g. \citep{shalev2014understanding}), but for completeness we provide a proof with explicit constants in \Cref{sec: VC_union}.

\begin{restatable}{lemma}{GeneralVCUnion}\label{lemma:general_VC_union}
Let $\mathcal{X}$ be a sample space, and let $\mathcal{H}_{1},...,\mathcal{H}_{M}$ be classes of hypotheses $h: \mathcal{X} \rightarrow \{0,1\}$ such that $\VCdim(\mathcal{H}_{i}) \le d$ for all $i \in [M]$. Let $\mathcal{H} := \mathcal{H}_{1} \cup ... \cup \mathcal{H}_{M}$. Then
\[
\VCdim(\mathcal{H}) \le 4d + 10 \ln M.
\]
\end{restatable}

\Cref{lemma:general_VC_union} has the following immediate corollary.

\begin{corollary}\label{cor:sparse_VC_union}
Let $\mathcal{X}$ be a sample space, and let $\mathcal{H} = (\mathcal{H}_{k})_{k \ge 1}$ be base hypothesis classes. Then for all $k$,
\[
\VCdim(\hat{\mathcal{H}}_{k}) \le 4 \cdot \VCdim(\mathcal{H}_{k}) + 10k \ln n.
\]
\end{corollary}

\begin{proof}
The class $\hat{\mathcal{H}}_{k}$ is the union of ${n \choose k} \le n^{k}$ copies of $\mathcal{H}_{k}$, one for each list of $k$ distinct features in $[n]$. The result now follows from \Cref{lemma:general_VC_union}.
\end{proof}

Given a hypothesis $h: \mathcal{X}^n \rightarrow \{0,1\}$, let $\Feat(h) \subseteq [n]$ be the set of features on which $h$ non-trivially depends. So in particular, if $h$ is $k$-sparse then $|\Feat(h)| \le k$.

We now present our first OOD generalization theorem that cares about feature structure. With sufficiently many samples, it guarantees with high probability that every $k$-sparse hypothesis in our hypothesis class that perfectly classifies the training samples has low error on the test distribution, provided that the training and test distributions induce the same distribution on the coordinates $A := \Feat(h)\cup\Feat(f)$. The dependence on $\Feat(h)$ is needed for the guarantee that every consistent hypothesis will generalize. Agreement only on $\Feat(f)$ is insufficient, as a consistent hypothesis might use an additional feature that was constant during training but changes at test time. For example, if $f(x)=x_1$ and $x_2=0$ throughout training, then $h(x)=x_1\oplus x_2$ perfectly predicts the training labels. If $x_2=1$ throughout testing, $h$ is always wrong, even when the marginal distribution of $x_1$ is unchanged.

The learner need not be given $\Feat(f)$ to achieve this guarantee. The theorem applies simultaneously to every hypothesis in our hypothesis class consistent with the training samples, so $h$ can be selected using those samples alone. For each selected hypothesis, transfer is guaranteed to every test distribution satisfying the condition on $\Feat(h)\cup\Feat(f)$. Different consistent hypotheses may therefore generalize to different families of test distributions. Sparsity limits the required agreement to at most $2k$ coordinates, while allowing arbitrary changes elsewhere.

\begin{theorem}\label{orig_feat_theorem}
Let $\mathcal{X}$ be a feature space, let $n$ be the number of features, let $\Dtrainx$ be a distribution over $\mathcal{X}^{n}$, and let $\epsilon, \delta > 0$. Let $\mathcal{H} = (\mathcal{H}_{k})_{k \ge 1}$ be base hypothesis classes over $\mathcal{X}$. Let $f: \mathcal{X}^{n} \rightarrow \{0,1\}$ be the ``ground truth," and assume $f \in \hat{\mathcal{H}}_{k}$ for some fixed sparsity $k$. Then there exists a constant $C$ such that, for all
\[
m \ge C \frac{\VCdim(\mathcal{H}_{k}) + k \log n}{\epsilon} \log \frac{1}{\delta\epsilon}
\]
the following holds. With probability at least $1-\delta$ over $x_{1},...,x_{m}$ drawn independently from $\Dtrainx$, any hypothesis $h \in \hat{\mathcal{H}}_{k}$ such that $h(x_{i}) = f(x_{i})$ for all $i \in [m]$ ``generalizes out of distribution," in the sense that
\[
\Pr_{x \sim \Dtestx}[h(x) = f(x)] \ge 1 - \epsilon,
\]
for all distributions $\Dtestx$ such that $\mathcal{D}_{A}' = \mathcal{D}_{A}$, where $A := \Feat(h) \cup \Feat(f)$.
\end{theorem}

\begin{proof}
Let $B := [n] - A$ so that we can represent any $x \in \mathcal{X}^{n}$ as $(x_{A}, x_{B})$, where $x_{A} = (x_{i})_{i \in A}$ and $x_{B} = (x_{i})_{i \in B}$. Also, given some $x_{A}$ let $\mathcal{D}_{B}(x_{A})$ and $\mathcal{D}_{B}'(x_{A})$ be the probability distributions over $x_{B}$ induced by $\Dtrainx$ and $\Dtestx$ respectively after we condition on $x_{A}$. Then we have
\begin{align*}
    \Pr_{x \sim \Dtestx}[h(x) = f(x)] &= \Pr_{\substack{x_{A} \sim \mathcal{D}_{A}' \\ x_{B} \sim \mathcal{D}_{B}'(x_{A})}}[h(x_{A}, x_{B}) = f(x_{A}, x_{B})] \\
    &= \Pr_{\substack{x_{A} \sim \mathcal{D}_{A} \\ x_{B} \sim \mathcal{D}_{B}'(x_{A})}}[h(x_{A}, x_{B}) = f(x_{A}, x_{B})] \\
    &= \Pr_{\substack{x_{A} \sim \mathcal{D}_{A} \\ x_{B} \sim \mathcal{D}_{B}(x_{A})}}[h(x_{A}, x_{B}) = f(x_{A}, x_{B})] \\
    &= \Pr_{x \sim \Dtrainx}[h(x) = f(x)]
\end{align*}
where the second line follows from $\mathcal{D}_{A} = \mathcal{D}_{A}'$, and the third line follows because neither $h$ nor $f$ depend on $x_{B}$. The result now follows from \Cref{thm:VCdim_PAC} combined with \Cref{cor:sparse_VC_union}.
\end{proof}

Finally, we combine Theorems \ref{thm:orig_alpha} and \ref{orig_feat_theorem} into a single statement.

\begin{theorem}\label{thm:orig_sparse_thm}
Let $\mathcal{X}$ be a feature space, let $n$ be the number of features, let $\Dtrainx$ be a distribution over $\mathcal{X}^{n}$, and let $\epsilon \in (0,1]$, $\delta > 0$, and $r \in (0,1]$. Let $\mathcal{H} = (\mathcal{H}_{k})_{k \ge 1}$ be base hypothesis classes over $\mathcal{X}$. Let $f: \mathcal{X}^{n} \rightarrow \{0,1\}$ be the ``ground truth," and assume $f \in \hat{\mathcal{H}}_{k}$ for some fixed sparsity $k$. Then there exists a constant $C$ such that, for all
\[
m \ge C \frac{\VCdim(\mathcal{H}_{k}) + k \log n}{r\epsilon} \log \frac{1}{\delta r\epsilon}
\]
the following holds. With probability at least $1-\delta$ over $x_{1},...,x_{m}$ drawn independently from $\Dtrainx$, any hypothesis $h \in \hat{\mathcal{H}}_{k}$ such that $h(x_{i}) = f(x_{i})$ for all $i \in [m]$ ``generalizes out of distribution," in the sense that
\[
\Pr_{x \sim \Dtestx}[h(x) = f(x)] \ge 1 - \epsilon
\]
for all distributions $\Dtestx$ such that $\rho_{\mathcal{D}_{A},\mathcal{D}_{A}'}(\epsilon) \ge r$, where $A := \Feat(h) \cup \Feat(f)$.
\end{theorem}

\section{Subspace Juntas}

Explaining out of distribution generalization in terms of sparse hypotheses can only give us basis-dependent explanations, however. To deal with this, we introduce the idea of subspace juntas, which instead of being dependent on a few features, is dependent on a low-dimensional subspace of the input. This is a natural generalization, since we recover sparse hypotheses by choosing the subspace to be the span of the standard basis vectors corresponding to the relevant features of the hypothesis. Subspace juntas are motivated by the idea that while the input space may be high-dimensional, we would like to find a low-dimensional relevant subspace hidden in the input space. For example, we would like to consider a neural network as a function of $W\x$ after the first layer of weights has been applied, which allows us to project from the higher $d$-dimensional input space to a $k$-dimensional subspace. The projection allows the class of these functions to be basis independent.

\begin{definition}[Subspace Junta]
    A function $f: \R^n \rightarrow \{0,1\}$ is a $k$-subspace junta (where $k \le n$) if there exists $W \in \R^{k \times n}$ and a function $g: \R^k \rightarrow \{0,1\}$ such that
    \[ f(\x) = f_W(\x) = g(W\x) \quad \forall \x \in \R^n. \]
\end{definition}

We would like to prove the analogous statements to the theorems for $k$-sparse hypotheses in the previous section. Instead of taking $A$ to be the union of features of $f$ and $h$, we let $f(x) = g^*(W^*x)$ and $h(x) = g(Wx)$, and take $A$ to be the span of the row vectors of $W$ and $W^*$. Then, we require that for the training distribution $\mathcal{D}$ and the test distribution $\mathcal{D}'$, their projections onto $A$, denoted as $\DtrainA$ and $\DtestA$, are equal or sufficiently close. We replace the VC-dimension in the lower bound on the number of required samples $m$ with the VC-dimension bound for $k$-subspace juntas instead of $k$-sparse hypotheses.

\begin{theorem}\label{thm:equal_dist_subspace_juntas}
    Let $\mathcal{X} = \R$ be the feature space, where $n$ is the number of features, $\Dtrainx$ a distribution on $\R^n$, and let $\epsilon, \delta > 0$. Let $\mathcal{F}$ be a class of $k$-subspace juntas over $\R^n$. Let $f: \R^n \rightarrow \{0,1\}$ be the ground truth function, and assume that $f \in \mathcal{F}$, with $f(x) = g^*(W^*x)$. Then there exists a constant $C$ such that for all
    \[
        m \ge C \cdot \frac{\VCdim(\mathcal{F})}{\epsilon} \log \frac{1}{\delta \epsilon},
    \]
    the following holds. With probability at least $1-\delta$ over $m$ samples $\x_1, \ldots, \x_m$ drawn independently from $\Dtrainx$, any hypothesis $h \in \mathcal{F}$, with $h(x) = g(Wx)$, such that $h(\x_i) = f(\x_i)$ for all $i \in [m]$ has low error on the test distribution, so that
    \[
        \Pr_{\x \sim \Dtestx} [h(\x) = f(\x)] \ge 1-\epsilon,
    \]
    for all distributions $\Dtestx$ such that the distribution of $P_A x$ where $\x\sim \Dtrainx$ is equal to the distribution of $P_A \x'$ where $\x'\sim \Dtestx$, where $A := \Span(w_1, \ldots, w_k, w_1^*, \ldots, w_k^*)$, and $P_A$ is the projection onto $A$.
    
\end{theorem}

\begin{proof}
    Since $A$ contains the span of the row vectors of $W$ and of the row vectors of $W^*$, then for any vector $x$ we have $Wx = W(P_Ax)$ and $W^*x = W^*(P_Ax)$. Then $h(x)$ and $f(x)$ depend only on the projection of $x$ onto $A$, so that $h(x) = g(Wx) = g(WP_Ax)$ and $f(x) = g^*(W^*x) = g^*(W^*P_Ax)$.

    Then we have
    \begin{align*}
        \Pr_{x\sim \Dtrainx}[h(x) = f(x)] &= \Pr_{x\sim \Dtrainx}[g(WP_Ax) = g^*(W^*P_Ax)] \\
        &= \Pr_{x' \sim \Dtestx}[g(WP_Ax') = g^*(W^*P_Ax')] \\
        &= \Pr_{x' \sim \Dtestx}[g(Wx') = g^*(W^*x')] \\
        &= \Pr_{x' \sim \Dtestx}[h(x') = f(x')].
    \end{align*}

    The second line follows from the assumption that the distribution of $P_A x$ where $x\sim \Dtrainx$ is equal to the distribution of $P_A x'$ where $x'\sim \Dtestx$, and the third line follows because $W$ and $W^*$ are each in the span of $A$.
    The result follows from \Cref{thm:VCdim_PAC} applied to $\mathcal{F}$.
\end{proof}

We can then combine \Cref{thm:equal_dist_subspace_juntas} with \Cref{thm:orig_alpha} to get a more general statement.

\begin{theorem}\label{thm:alpha_dist_subspace_juntas}
    Let $\mathcal{X} = \R$ be the feature space, where $n$ is the number of features, $\Dtrainx$ a distribution on $\R^n$, and let $\epsilon \in (0,1]$, $\delta > 0$, and $r \in (0,1]$. Let $\mathcal{F}$ be a class of $k$-subspace juntas over $\R^n$. Let $f: \R^n \rightarrow \{0,1\}$ be the ground truth function, and assume that $f \in \mathcal{F}$, with $f(x) = g^*(W^*x)$. Then there exists a constant $C$ such that for all

    \[
        m \ge C \cdot \frac{\VCdim(\mathcal{F})}{r\epsilon} \log \frac{1}{\delta r\epsilon},
    \]
    the following holds. With probability at least $1-\delta$ over $m$ samples $x_1, \ldots, x_m$ drawn independently from $\Dtrainx$, any hypothesis $h \in \mathcal{F}$, with $h(x) = g(Wx)$, such that $h(x_i) = f(x_i)$ for all $i \in [m]$ has low error on the test distribution, so that

    \[
        \Pr_{x \sim \Dtestx} [h(x) = f(x)] \ge 1-\epsilon,
    \]
    for all distributions $\Dtestx$ such that $\rho_{\DtrainA, \DtestA}(\epsilon) \ge r$, where $A := \Span(w_1, \ldots, w_k, w_1^*, \ldots, w_k^*)$.
\end{theorem}
    
\begin{proof}
Since $h$ and $f$ only depend on the subspace $A$, there exist functions $h_A, f_A: A \to \{0,1\}$ such that $h_A(x_A) = h(x)$ and $f_A(x_A) = f(x)$ for all $x \in \R^n$, where $x_A := P_Ax$. Thus $\Pr_{x' \sim \Dtestx}[h(x') \ne f(x')] = \Pr_{x_A' \sim \DtestA}[h_A(x_A') \ne f_A(x_A')]$, and similarly for $\Dtrainx$.
Then by the definition of $\rho_{\DtrainA,\DtestA}(\epsilon)$ and the assumption $\rho_{\DtrainA,\DtestA}(\epsilon) \ge r$, we have
\[\Pr_{x_A' \sim \DtestA}[h_A(x_A') \ne f_A(x_A')] \ge \epsilon \implies
\Pr_{x_A \sim \DtrainA}[h_A(x_A) \ne f_A(x_A)] \ge r\epsilon.
\]
Using the same classical argument as in the proof of \Cref{thm:VCdim_PAC}, with $r\epsilon$ in place of $\epsilon$, we obtain the result.
\end{proof}

\paragraph{Examples of $k$-subspace juntas.}
If we allow $g$ to come from the class $\classG$ of all functions from $\R^k \rightarrow \{0,1\}$, the VC-dimension of the class $\mathcal{F}$ of $k$-subspace juntas would be infinite, rendering \Cref{thm:alpha_dist_subspace_juntas} meaningless. Therefore, we would like to give the analogous statement to \Cref{cor:sparse_VC_union} by answering the following question:

\begin{question}\label{q:general_subspace_bound}
    Let $\mathcal{F}$ be a class of $k$-subspace juntas such that for each function $f: \R^n \rightarrow \{0,1\}$ in $\mathcal{F}$, the associated function $g: \R^k \rightarrow \{0,1\}$ such that $f(\x) = g(W\x)$ belongs to a class $\classG$ with VC-dimension $\le d$. What is an upper bound on the VC-dimension of $\mathcal{F}$ in terms of $n, k, d$?
\end{question}

Note that we cannot directly apply \Cref{lemma:general_VC_union}, as we would be taking an infinite union corresponding to the arbitrary choice for $W$.
There are some instantiations of $\classG$ where we can get VC-dimension bounds following from those of well known classes.
For instance, by taking $\classG$ to be the class of halfspaces in $\R^k$, which has VC-dimension $k+1$, we have a bound on the VC-dimension of $\mathcal{F}$ because it is the set of halfspaces in $\R^n$, which has VC-dimension $n+1$. 

We can also consider other classes for $\classG$ such as polynomial threshold functions (PTFs). We can classically get a VC-dimension bound for PTFs of degree $p$ by considering the monomials of the polynomial as new features (feature expansion), so that we can view it as a halfspace in the new features. Then $\classG$ would be the class of degree $p$ PTFs in $\R^k$ and would have VC-dimension at most ${k+p \choose p} + 1$, while $\mathcal{F}$ would be a subset of degree $p$ PTFs in $\R^n$ and would have VC-dimension at most ${n+p \choose p}+1$. 

However, it turns out that there is a counterexample to the general statement of \Cref{q:general_subspace_bound} based on a classical construction (see \citep[Chapter 7.2]{anthony2009neural} for a similar example). We leave its proof to the appendix.

\begin{restatable}{claim}{Counterexample}\label{cl:counterexample}
    Let $\classG = \{g\}$ where $g: \R \rightarrow \{0,1\}$ and 
    \[g(x) = 
    \begin{cases} 
          1 & \text{if } \lfloor x \rfloor \equiv 1 \pmod{2} \\
          0 & \text{otherwise.}
    \end{cases}\]
    Let $\mathcal{F}$ be the class consisting of all functions of the form $f(x) = g(w \cdot x)$ where $w \in \R^n$ and $\Norm{w} = 1$. Then the VC-dimension of $\classG$ is $0$, and the VC-dimension of $\mathcal{F}$ is $\infty$.
\end{restatable}

The above example demonstrates that subspace juntas may have infinite VC-dimension even if the inner class $\classG$ has bounded VC-dimension. This is unavoidable: using similar constructions, it is known that there are neural networks whose activation functions are smooth, bounded, and monotonically increasing, which have infinite VC-dimension. Indeed, there are examples where the network only has two layers and two neurons in the first layer, which has activation functions satisfying being convex to the left of zero and concave to the right in addition to the properties listed above which still has infinite VC-dimension \citep[Chapter 7.2]{anthony2009neural}. 

While bounding the VC dimension of $\classG$ is not enough, we would like an expressive choice for $\classG$ such that the corresponding class of subspace juntas $\mathcal{F}$ still has finite VC-dimension. A natural choice is to take $\classG$ to be a class of (indicators of) semi-algebraic sets. Any neural network with piecewise polynomial activations (including halfspaces and ReLUs), when used as a classifier, belongs to one of these classes.

\begin{definition}[Semi-algebraic set, see e.g. \citep{Chernikov2017ModelTheory}, \citep{anthony2009neural}]
    Let 
        \begin{equation*}
                \mathcal{S}^b_{t,\ell,n} := \{ g: \mathbb{R}^n \rightarrow \{0,1\} :
                g(x) = b(p_1(x), \ldots, p_t(x)) \},
        \end{equation*}
    where $b : \{0,1\}^t \rightarrow \{0,1\}$ is an arbitrary but fixed Boolean function, and $p_i : \R^n \rightarrow \{0,1\}$ are degree $\le \ell$ polynomial threshold functions. That is, $p_i(x) = \sign(q_i(x))$ for some real polynomial $q_i: \R^n \rightarrow \R$ with degree $\le \ell$. Each $g \in \mathcal{S}^b_{t,\ell,n}$ is the indicator of a semi-algebraic subset of $\R^n$.
\end{definition}

\begin{theorem}[\cite{Chernikov2017ModelTheory}, Chapter 1]\label{thm: VCD_semi}

    Let $\mathcal{S}^b_{t,\ell,n}$ be a class of indicators of semi-algebraic sets as defined above. Then, its VC-dimension can be upper bounded as follows:
    \[
        \VCdim(\mathcal{S}^b_{t,\ell,n}) \le 2t {n+\ell \choose \ell}\log \Paren{t(t+1) {n+\ell \choose \ell}}.
    \]
\end{theorem}

Thus, the VC-dimension bound is polynomial in $t$ and $n$ for some fixed $\ell$. Now, we consider the definition of subspace juntas with respect to semi-algebraic sets.

\begin{definition}
    The class of semi-algebraic subspace juntas is written as
    \begin{equation*}
        \mathcal{F}^b_{n,\ell,k,t} := \{ f: \mathbb{R}^n \rightarrow \{0,1\} : f(x) = b(p_1(Wx), \dots, p_t(Wx)) \}
    \end{equation*}
    where $b : \{0,1\}^t \rightarrow \{0,1\}$ is an arbitrary but fixed Boolean function, $p_i : \mathbb{R}^k \rightarrow \{0,1\}$ are degree $\le \ell$ polynomial threshold functions, and $W \in \mathbb{R}^{k \times n}$ is an arbitrary real matrix.
\end{definition}

Note that $\mathcal{F}^b_{n,\ell,k,t}$ is a subset of $\mathcal{S}^b_{t,\ell,n}$, because we can consider any degree $\ell$ polynomial $p(Wx)$ as a degree $\ell$ polynomial $q(x)$. Thus $\mathcal{F}^b_{n,\ell,k,t}$ is a class of indicators of semi-algebraic sets, so its VC-dimension is also bounded by \Cref{thm: VCD_semi}. This establishes that $\mathcal{F}^b_{n,\ell,k,t}$ has finite VC dimension, but we can obtain the following tighter bound which has an explicit dependence on $k$ and a linear dependence on $n$. We leave its proof to the appendix.

\begin{theorem}\label{thm:VCD_semisubspace}
    Let $\mathcal{F}^b_{n,\ell,k,t}$ be a class of semi-algebraic subspace juntas as defined above. Then, its VC-dimension can be upper bounded as follows:
    \[
        \VCdim(\mathcal{F}^b_{n,\ell,k,t}) \le 2\Paren{kn + t{{k+\ell} \choose \ell}} \log(12t(\ell+1)).
    \]
\end{theorem}

This bound is linear in $n$ rather than scaling like $n^\ell$. Since an upper bound on VC-dimension corresponds to a smaller sample complexity required to learn a class, we see that the case $k \ll n$ is useful in efficiently learning subspace juntas even in spaces with large input dimension.

We note that the linear dependence on $n$ of the VC-dimension bound for $\mathcal{F}^b_{n,\ell,k,t}$ is unavoidable. For example, take $\ell=1, k = 1, t = 1$, and $b: \{0,1\} \rightarrow \{0,1\}$ to be the identity function. Consider the functions $f \in \mathcal{F}^b_{n,1,1,1}$ of the form $f(x) = \sign(w \cdot x)$ for an arbitrary $w \in \R^n$. Then these form the class of all homogeneous half-spaces in $\R^n$, which have known VC dimension $n$. Since this is a subset of $\mathcal{F}^b_{n,1,1,1}$, the overall class has a VC-dimension lower bounded by $n$.

\section{Agnostic Generalization}\label{sec:agnostic}

So far we have worked in the realizable setting where the training and test labels are determined by a single ground truth function. In this section we consider the agnostic setting where we have labeled training and test distributions $\Dtrain$ and $\Dtest$ such that the label $y$ on any given point $(x,y)$ may be arbitrary. We cannot hope to achieve the same error guarantees as the realizable setting when there may not be any function in the class which closely fits potentially arbitrary labels, so we adopt standard terms from domain adaptation used for error bounds such as $\lambda$, the minimum sum of training and test error for a single function $f \in \mathcal{F}$. Other relevant definitions can be found in the preliminaries section.

\subsection{Sparse Hypotheses}

We use the weight ratio $\rho$ from \Cref{def:rho} to allow approximate overlap of the relevant marginals.

\begin{theorem}\label{thm:agnostic_rho_sparse}
Let $\mathcal{X}$ be a feature space, let $n$ be the number of features, let $\Dtrain$ be a labeled distribution over $\mathcal{X}^n \times \{0,1\}$, and let $\epsilon, \delta > 0$ and $r \in (0,1]$. Let $\mathcal{H} = (\mathcal{H}_k)_{k \ge 1}$ be base hypothesis classes over $\mathcal{X}$. Let $f^* \in \hat{\mathcal{H}}_k$ for some fixed sparsity $k$ minimize the sum of training and test error. Then there exists a constant $C$ such that, for all
\[
m \ge C \frac{\VCdim(\mathcal{H}_k) + k\log n + \log(1/\delta)}{r^2\epsilon^2},
\]
the following holds. With probability at least $1-\delta$ over a sample $S = ((x_i,y_i))_{i=1}^m$ drawn independently from $\Dtrain$, any hypothesis $h \in \hat{\mathcal{H}}_k$ minimizing empirical training error ``generalizes out of distribution," in the sense that
\[
\err_{\Dtest}(h) \le \frac{\opt + \lambda}{r} + \epsilon,
\]
for all labeled distributions $\Dtest$ such that $\rho_{\DtrainA,\DtestA}(\epsilon) \ge r$, where $A := \Feat(h) \cup \Feat(f^*)$.
\end{theorem}

We note that the error bound in this theorem generalizes the previous bounds. When $\DtrainA = \DtestA$, we have $\rho_{\DtrainA,\DtestA}(\epsilon) = 1$, so taking $r=1$ recovers the equal-marginal guarantee $\err_{\Dtest}(h) \le \opt + \lambda + \epsilon$, with sample complexity $m \ge C(\VCdim(\mathcal{H}_k) + k\log n + \log(1/\delta))/\epsilon^2$. Furthermore, in the realizable case, $\opt = \lambda = 0$, since the ground truth function $f \in \hat{\mathcal{H}}_k$ achieves zero training and test error. The bound therefore reduces to $\err_{\Dtest}(h) \le \epsilon$, recovering the error guarantee of \Cref{orig_feat_theorem,thm:orig_sparse_thm}.

\begin{proof}
We bound test error with the triangle inequality: $\err_{\Dtest}(h) \le \Pr_{x \sim \Dtestx}[h(x) \ne f^*(x)] + \err_{\Dtest}(f^*).$
We will relate the test disagreement probability to the training disagreement probability using $\rho_{\DtrainA,\DtestA}(\epsilon)$.

Let $\eta := r\epsilon$. By standard agnostic PAC learning \citep[Chapters 6 and 28]{shalev2014understanding} and the VC-dimension bound in \Cref{cor:sparse_VC_union}, the stated sample complexity guarantees that, with probability at least $1-\delta$, every hypothesis $h \in \hat{\mathcal{H}}_k$ minimizing empirical training error satisfies $\err_{\Dtrain}(h) \le \opt + \eta.$
If $h(x) \ne f^*(x)$, then at least one of $h(x) \ne y$ and $f^*(x) \ne y$ holds. Hence,
\[
\Pr_{x \sim \Dtrainx}[h(x) \ne f^*(x)]
\le \err_{\Dtrain}(h) + \err_{\Dtrain}(f^*) \le \opt + \eta + \err_{\Dtrain}(f^*).
\]

Since neither $h$ nor $f^*$ depends on coordinates outside $A$, we may write $h(x) = h_A(x_A)$ and $f^*(x) = f_A^*(x_A)$. Thus $\Pr_{x \sim \Dtrainx}[h(x) \ne f^*(x)] = \Pr_{x_A \sim \DtrainA}[h_A(x_A) \ne f_A^*(x_A)]$, and similarly for $\Dtest$.

If $\Pr_{x_A \sim \DtestA}[h_A(x_A) \ne f_A^*(x_A)] \ge \epsilon$, the disagreement event on $A$ that $h_A(x_A) \ne f_A^*(x_A)$ is among the events in the infimum defining $\rho_{\DtrainA,\DtestA}(\epsilon)$. Therefore,
\[
\Pr_{x_A \sim \DtestA}[h_A(x_A) \ne f_A^*(x_A)] \ge \epsilon
\implies \frac{\Pr_{x_A \sim \DtrainA}[h_A(x_A) \ne f_A^*(x_A)]}
{\Pr_{x_A \sim \DtestA}[h_A(x_A) \ne f_A^*(x_A)]}
\ge \rho_{\DtrainA,\DtestA}(\epsilon) \ge r.
\]
In particular, this also means that $\Pr_{x \sim \Dtestx}[h(x) \ne f^*(x)] \ge \epsilon $ implies that $ \Pr_{x \sim \Dtrainx}[h(x) \ne f^*(x)] \ge r\Pr_{x \sim \Dtestx}[h(x) \ne f^*(x)].$ Thus, whether the test disagreement probability is below $\epsilon$ or at least $\epsilon$, we have
\[
\Pr_{x \sim \Dtestx}[h(x) \ne f^*(x)] \le \max\left\{\epsilon,
\frac{\Pr_{x \sim \Dtrainx}[h(x) \ne f^*(x)]}{r}\right\} \le \max\left\{\epsilon,
\frac{\opt + \eta + \err_{\Dtrain}(f^*)}{r}\right\}.
\]
The maximum accounts for the fact that $\rho_{\DtrainA,\DtestA}(\epsilon)$ only considers events with test probability at least $\epsilon$; the $\epsilon$ in the maximum bounds the other case. With $\eta = r\epsilon$, the second term in the maximum is
\[
\frac{\opt + \eta + \err_{\Dtrain}(f^*)}{r}
= \frac{\opt + \err_{\Dtrain}(f^*)}{r} + \epsilon
\ge \epsilon.
\]
Substituting into the triangle inequality from the beginning, we have:
\begin{align*}
\err_{\Dtest}(h)
&\le \frac{\opt + \err_{\Dtrain}(f^*)}{r}
+ \epsilon + \err_{\Dtest}(f^*) \\
&\le \frac{\opt + \err_{\Dtrain}(f^*) + \err_{\Dtest}(f^*)}{r}
+ \epsilon \\
&= \frac{\opt + \lambda}{r} + \epsilon.
\end{align*}
The second inequality follows because $r \le 1$, and the final equality uses the definition of $\lambda$ and the choice of $f^*$.
\end{proof}

\subsection{Subspace Juntas}
For subspace juntas, the same argument applies to the projected distributions on the subspace used by $h$ and $f^*$.

\begin{theorem}\label{thm:agnostic_rho_subspace_juntas}
Let $\mathcal{X} = \R$ be the feature space, where $n$ is the number of features, $\Dtrain$ a labeled distribution on $\R^n \times \{0,1\}$, and let $\epsilon, \delta > 0$ and $r \in (0,1]$. Let $\mathcal{F}$ be a class of $k$-subspace juntas over $\R^n$. Let $f^*(x) = g^*(W^*x) \in \mathcal{F}$ minimize the sum of training and test error. Then there exists a constant $C$ such that, for all
\[
m \ge C \frac{\VCdim(\mathcal{F}) + \log(1/\delta)}{r^2\epsilon^2},
\]
the following holds. With probability at least $1-\delta$ over a sample $S = ((x_i,y_i))_{i=1}^m$ drawn independently from $\Dtrain$, any hypothesis $h(x) = g(Wx) \in \mathcal{F}$ minimizing empirical training error has low error on the test distribution, so that
\[
\err_{\Dtest}(h) \le \frac{\opt + \lambda}{r} + \epsilon,
\]
for all labeled distributions $\Dtest$ such that $\rho_{\DtrainA,\DtestA}(\epsilon) \ge r$, where $A := \Span(w_1,\allowbreak \ldots,\allowbreak w_k,\allowbreak w_1^*,\allowbreak \ldots,\allowbreak w_k^*)$, and $P_A$ is the projection onto $A$.
\end{theorem}

When $\DtrainA = \DtestA$, taking $r=1$ again gives $\err_{\Dtest}(h) \le \opt + \lambda + \epsilon$, now with sample complexity $m \ge C(\VCdim(\mathcal{F}) + \log(1/\delta))/\epsilon^2$.

\begin{proof}
Similar to the realizable case for subspace juntas, since $h$ and $f^*$ only depend on the subspace $A$, there exist functions $h_A, f_A^*: A \to \{0,1\}$ such that $h_A(x_A) = h(x)$ and $f_A^*(x_A) = f^*(x)$ for all $x \in \R^n$, where $x_A := P_Ax$. Thus $\Pr_{x \sim \Dtrainx}[h(x) \ne f^*(x)] = \Pr_{x_A \sim \DtrainA}[h_A(x_A) \ne f_A^*(x_A)]$, and similarly for $\Dtestx$. Using the standard agnostic PAC error guarantee for $\mathcal{F}$ with accuracy $r\epsilon$, the same argument as in \Cref{thm:agnostic_rho_sparse} gives the desired result.
\end{proof}

\section{Open Problems}

We have proposed a PAC-style framework for explaining out of distribution generalization through sparsity. This formalizes Occam's Razor as a natural inductive bias, explaining when we can expect generalization even when the training and test distributions differ entirely on irrelevant coordinates. Are there other natural generalizations of the PAC learning theory or VC-dimension that have explanatory power for OOD generalization? We also note the following directions for future work:

\paragraph{Improved VC-dimension bound for semi-algebraic subspace juntas.}
Our current bound on VC-dimension for semi-algebraic subspace juntas is a specific instantiation of a more general VC bound. Is it possible to obtain a tighter bound for the VC-dimension of semi-algebraic subspace juntas? Can a matching lower bound be obtained? Is the dependence on $k$ optimal?

\paragraph{Incorporating max margin into OOD.}

From the perceptron learning algorithm and the perceptron convergence theorem to SVMs, the concept of maximum margin has been foundational and essential in machine learning. Even in more recent work on the softmax-attention model in transformer architecture shows implicit bias towards a max-margin solution \citep{TLZOS23maxmargin}. Can we formally determine the role of max-margin in OOD generalization?

\paragraph{Explicit algorithms.}
Our theorems primarily answer statistical learning theory questions, addressing when we have enough samples and under which distributional assumptions we can learn sparse hypotheses and subspace juntas. A natural direction for future work would be to find explicit algorithms for recovering the relevant subspace.

\section{AI Disclosure}
ChatGPT was used to help identify useful definitions and constructions and their references, including the classical construction in \cite{anthony2009neural}, using the $\epsilon$-weight ratio from \cite{bendavid2012quantity} in place of a related definition the authors initially had thought of, and the semi-algebraic sets and their VC dimension bounds. It aided in improving wording and formatting the LaTeX. Claude was used to write the code and make plots for the experiments that the authors had set up.

\section*{Acknowledgments}
This work was begun when one of us (SA) was at OpenAI. We are grateful to Coefficient Giving for supporting this research. We thank Ilya Sutskever, Lionel Levine, Adam Klivans, Eric Price, Harvey Lederman, and Boaz Barak for helpful discussions. The authors declare no competing interests.

\bibliographystyle{alpha}
\bibliography{references}

@article{valiant1984,
author = {Valiant, L. G.},
title = {A theory of the learnable},
year = {1984},
issue_date = {Nov. 1984},
publisher = {Association for Computing Machinery},
address = {New York, NY, USA},
volume = {27},
number = {11},
issn = {0001-0782},
url = {https://doi.org/10.1145/1968.1972},
doi = {10.1145/1968.1972},
journal = {Commun. ACM},
month = nov,
pages = {1134–1142},
numpages = {9}
}

@book{goodman1955,
  title     = {Fact, Fiction, and Forecast},
  author    = {Goodman, Nelson},
  year      = {1955},
  publisher = {Harvard University Press},
  address   = {Cambridge, MA}
}

@article{blumer1989,
author = {Blumer, Anselm and Ehrenfeucht, A. and Haussler, David and Warmuth, Manfred K.},
title = {Learnability and the Vapnik-Chervonenkis dimension},
year = {1989},
issue_date = {Oct. 1989},
publisher = {Association for Computing Machinery},
address = {New York, NY, USA},
volume = {36},
number = {4},
issn = {0004-5411},
url = {https://doi.org/10.1145/76359.76371},
doi = {10.1145/76359.76371},
journal = {J. ACM},
month = oct,
pages = {929–965},
numpages = {37}
}

@article{bendavid2010theory,
  title     = {A theory of learning from different domains},
  author    = {Ben-David, Shai and Blitzer, John and Crammer, Koby and Kulesza, Alex and Pereira, Fernando and Vaughan, Jennifer Wortman},
  journal   = {Machine Learning},
  volume    = {79},
  number    = {1},
  pages     = {151--175},
  year      = {2010},
  publisher = {Springer},
  doi       = {10.1007/s10994-009-5152-4},
  url       = {https://doi.org/10.1007/s10994-009-5152-4}
}

@book{anthony2009neural,
  title     = {Neural Network Learning: Theoretical Foundations},
  author    = {Anthony, Martin and Bartlett, Peter L.},
  year      = {2009},
  publisher = {Cambridge University Press},
  address   = {Cambridge, UK},
  isbn      = {978-0-521-11862-0},
  note      = {Originally published in 1999}
}

@unpublished{Chernikov2017ModelTheory,
  author = {Artem Chernikov},
  title  = {Model Theory and Combinatorics: Chapter 2 (Draft)},
  note   = {Lecture notes, last updated February 28, 2017},
  year   = {2017},
  url    = {http://www.math.ucla.edu/~chernikov/teaching/StabilityTheory285D/StabilityNotes.pdf}
}

@inproceedings{mmr09,
title	= {Domain Adaptation: Learning Bounds and Algorithms},author	= {Yishay Mansour and Mehryar Mohri and Afshin Rostamizadeh},year	= {2009},URL	= {http://www.cs.nyu.edu/~mohri/postscript/nadap.pdf},booktitle	= {Proceedings of The 22nd Annual Conference on Learning Theory (COLT 2009)},address	= {Montr\'eal, Canada}}

@inproceedings{Blitzer07,
 author = {Blitzer, John and Crammer, Koby and Kulesza, Alex and Pereira, Fernando and Wortman, Jennifer},
 booktitle = {Advances in Neural Information Processing Systems},
 editor = {J. Platt and D. Koller and Y. Singer and S. Roweis},
 pages = {},
 publisher = {Curran Associates, Inc.},
 title = {Learning Bounds for Domain Adaptation},
 url = {https://proceedings.neurips.cc/paper_files/paper/2007/file/42e77b63637ab381e8be5f8318cc28a2-Paper.pdf},
 volume = {20},
 year = {2007}
}

@inproceedings{bendavid2006analysis,
  title     = {Analysis of Representations for Domain Adaptation},
  author    = {Ben-David, Shai and Blitzer, John and Crammer, Koby and Pereira, Fernando},
  booktitle = {Advances in Neural Information Processing Systems 19 (NIPS 2006)},
  pages     = {137--144},
  year      = {2006},
  publisher = {MIT Press},
  url       = {https://papers.nips.cc/paper/2006/hash/b1b043238319d2da773d191a2046ca7c-Abstract.html}
}

@inproceedings{bendavid2010impossibility,
  title     = {Impossibility Theorems for Domain Adaptation},
  author    = {Ben-David, Shai and Lu, Tyler and Luu, Teresa and P{\'a}l, D{\'a}vid},
  booktitle = {Proceedings of the Thirteenth International Conference on Artificial Intelligence and Statistics (AISTATS 2010)},
  series    = {JMLR Workshop and Conference Proceedings},
  volume    = {9},
  pages     = {129--136},
  year      = {2010},
  publisher = {JMLR.org},
  url       = {http://proceedings.mlr.press/v9/ben-david10a.html}
}

@inproceedings{redko2017theoretical,
  title     = {Theoretical Analysis of Domain Adaptation with Optimal Transport},
  author    = {Redko, Ievgen and Habrard, Amaury and Sebban, Marc},
  booktitle = {Machine Learning and Knowledge Discovery in Databases: European Conference, ECML PKDD 2017},
  pages     = {737--753},
  year      = {2017},
  publisher = {Springer International Publishing},
  doi       = {10.1007/978-3-319-71249-9_44}
}

@article{redko2020survey,
  title     = {A Survey on Domain Adaptation Theory: Learning Bounds and Theoretical Guarantees},
  author    = {Redko, Ievgen and Morvant, Emilie and Habrard, Amaury and Sebban, Marc and Bennani, Youn{\`e}s},
  journal   = {arXiv preprint arXiv:2004.11829},
  year      = {2020},
  url       = {https://arxiv.org/abs/2004.11829}
}

@article{arjovsky2019invariant,
  title={Invariant Risk Minimization},
  author={Arjovsky, Martin and Bottou, L{\'e}on and Gulrajani, Ishaan and Lopez-Paz, David},
  journal={arXiv preprint arXiv:1907.02893},
  year={2019}
}

@article{mossel04juntas,
  title   = {Learning Functions of $k$ Relevant Variables},
  author  = {Mossel, Elchanan and O'Donnell, Ryan and Servedio, Rocco A.},
  journal = {Journal of Computer and System Sciences},
  volume  = {69},
  number  = {3},
  pages   = {421--434},
  year    = {2004},
  note    = {Preliminary version in STOC 2003},
  url     = {https://www.cs.columbia.edu/~rocco/papers/stoc03juntas.html}
}

@techreport{blum1994relevant,
  author      = {Blum, Avrim L.},
  title       = {Relevant Examples and Relevant Features: Thoughts from Computational Learning Theory},
  institution = {School of Computer Science, Carnegie Mellon University},
  year        = {1994},
  number      = {FS-94-02},
  series      = {AAAI Technical Report},
  address     = {Pittsburgh, PA},
  pages       = {14--18}
}

@article{blum1997learning,
  author    = {Blum, Avrim L. and Kannan, Ravindran},
  title     = {Learning an Intersection of a Constant Number of Halfspaces over a Uniform Distribution},
  journal   = {Journal of Computer and System Sciences},
  volume    = {54},
  number    = {3},
  pages     = {371--380},
  year      = {1997},
  publisher = {Academic Press},
  doi       = {10.1006/jcss.1997.1475}
}

@inproceedings{vempala2013complexity,
  author    = {Vempala, Santosh and Xiao, Ying},
  booktitle = {2013 Asilomar Conference on Signals, Systems and Computers}, 
  title     = {Complexity of learning subspace juntas and ICA}, 
  year      = {2013},
  pages     = {182-186},
  doi       = {10.1109/ACSSC.2013.6810286},
  publisher = {IEEE},
  address   = {Pacific Grove, CA, USA}
}

@book{shalev2014understanding,
  title={Understanding machine learning: From theory to algorithms},
  author={Shalev-Shwartz, Shai and Ben-David, Shai},
  year={2014},
  publisher={Cambridge university press},
  address={Cambridge, UK}
}

@article{zhou2022dg_survey,
  title   = {Domain Generalization: A Survey},
  author  = {Zhou, Kaiyang and Liu, Ziwei and Qiao, Yu and Xiang, Tao and Loy, Chen Change},
  journal = {IEEE Transactions on Pattern Analysis and Machine Intelligence},
  year    = {2022},
  doi     = {10.1109/TPAMI.2022.3195549},
  eprint  = {2103.02503},
  archivePrefix = {arXiv}
}

@article{wang2023dg_survey,
  title   = {Generalizing to Unseen Domains: A Survey on Domain Generalization},
  author  = {Wang, Jindong and Lan, Cuiling and Liu, Chang and Ouyang, Yidong and Qin, Tao and Lu, Wang and Chen, Yiqiang and Zeng, Wenjun and Yu, Philip S.},
  journal = {IEEE Transactions on Knowledge and Data Engineering},
  year    = {2023},
  volume  = {35},
  number  = {8},
  pages   = {8052--8072},
  doi     = {10.1109/TKDE.2022.3178128},
  eprint  = {2103.03097},
  archivePrefix = {arXiv}
}

@inproceedings{gulrajani2021domainbed,
  title     = {In Search of Lost Domain Generalization},
  author    = {Gulrajani, Ishaan and Lopez-Paz, David},
  booktitle = {International Conference on Learning Representations (ICLR)},
  year      = {2021},
  url       = {https://openreview.net/forum?id=lQdXeXDoWtI}
}

@inproceedings{sagawa2019groupdro,
  title={Distributionally Robust Neural Networks},
  author={Sagawa, Shiori and Koh, Pang Wei and Hashimoto, Tatsunori B and Liang, Percy},
  booktitle={International Conference on Learning Representations},
  year = {2020}
}

@inproceedings{Ye2021Towards,
  title     = {Towards a Theoretical Framework of Out-of-Distribution Generalization},
  author    = {Ye, Haotian and Xie, Chuanlong and Cai, Tianle and Li, Ruichen and Li, Zhenguo and Wang, Liwei},
  booktitle = {Advances in Neural Information Processing Systems},
  volume    = {34},
  year      = {2021},
  url       = {https://proceedings.neurips.cc/paper/2021/hash/c5c1cb0bebd56ae38817b251ad72bedb-Abstract.html}
}

@inproceedings{TLZOS23maxmargin,
author = {Tarzanagh, Davoud Ataee and Li, Yingcong and Zhang, Xuechen and Oymak, Samet},
title = {Max-margin token selection in attention mechanism},
year = {2023},
publisher = {Curran Associates Inc.},
address = {Red Hook, NY, USA},
booktitle = {Proceedings of the 37th International Conference on Neural Information Processing Systems},
articleno = {2096},
numpages = {49},
location = {New Orleans, LA, USA},
series = {NIPS '23}
}

@article{AV06projection,
author = {Arriaga, Rosa I. and Vempala, Santosh},
title = {An algorithmic theory of learning: Robust concepts and random projection},
year = {2006},
issue_date = {May       2006},
publisher = {Kluwer Academic Publishers},
address = {USA},
volume = {63},
number = {2},
issn = {0885-6125},
url = {https://doi.org/10.1007/s10994-006-6265-7},
doi = {10.1007/s10994-006-6265-7},
journal = {Mach. Learn.},
month = may,
pages = {161–182},
numpages = {22}
}

@misc{ding2022retiringadultnewdatasets,
      title={Retiring Adult: New Datasets for Fair Machine Learning}, 
      author={Frances Ding and Moritz Hardt and John Miller and Ludwig Schmidt},
      year={2022},
      eprint={2108.04884},
      archivePrefix={arXiv},
      primaryClass={cs.LG},
      url={https://arxiv.org/abs/2108.04884}, 
}

@inproceedings{bendavid2012quantity,
  author    = {Ben-David, Shai and Shalev-Shwartz, Shai and Urner, Ruth},
  title     = {Domain Adaptation---Can Quantity Compensate for Quality?},
  booktitle = {International Symposium on Artificial Intelligence and Mathematics (ISAIM)},
  year      = {2012},
  url       = {https://cs.uwaterloo.ca/~shai/ISAIM12.pdf}
}

\appendix
\newpage

\section{\texorpdfstring{Proof of Lemma~\ref{lemma:general_VC_union}}{Proof of Lemma 2}}\label{sec: VC_union}

\GeneralVCUnion*

\begin{proof}
Let $D = \VCdim(\mathcal{H})$. Set $L := \frac{D}{2}-d$, so that $D = 2(d+L)$, and note that we can assume $L \ge 0$, since otherwise we are done. By definition, there exist sample points $x_{1},...,x_{D} \in \mathcal{X}$ such that for all $2^{D}$ possible strings $y = y_{1} \cdot\cdot\cdot y_{D} \in \{0,1\}^{D}$, there exists a hypothesis $h_{y} \in \mathcal{H}$ such that $h_{y}(x_{1}) = y_{1},...,h_{y}(x_{D}) = y_{D}$. By counting, clearly there exists an $i^{*}$ such that $h_{y} \in \mathcal{H}_{i^{*}}$ for at least $2^{D}/M$ values of $y$. So then
\[
\frac{2^{D}}{M} \le \sum_{i=0}^{d} \binom{D}{i} \le 2^{D} \exp\left(-\frac{L^{2}}{2D}\right)
\]
where the first inequality uses \Cref{lemma:Sauer} (Sauer's Lemma) and the second uses Chernoff. Solving, we get $L \le \sqrt{2D \ln M}$, so
\[
\frac{D}{2}-d \le \sqrt{2D \ln M} ,
\]
or solving the quadratic and using the arithmetic-geometric mean inequality,
\[
D \le 2d + 4 \ln M + 4 \sqrt{\ln^{2} M + d \ln M} \le 4d + 10 \ln M.
\]
\end{proof}

\section{\texorpdfstring{Proof of Claim \ref{cl:counterexample}}{Proof of Claim 1}}
\Counterexample*
\begin{figure}[h]
    \centering
    \begin{tikzpicture}
        \begin{axis}[
            axis lines = middle,
            xmin = -5.5, xmax = 5.5,
            ymin = -0.2, ymax = 2.5,
            ytick = {0,1},
            yticklabel style={xshift=0.5cm},
            xtick = {-5,-4,...,5},
            grid = both,
            grid style = {solid, gray!15},
            clip = false, 
            width = 10cm, 
            height = 6cm,
            tick label style={font=\small},
            label style={font=\small},
            xlabel = {$x$},
            ylabel = {$y$}
        ]
        
        \addplot [
            domain=-5:5, 
            samples=200, 
            jump mark left, 
            very thick, 
            red,
            no marks
        ] {mod(floor(x), 2) == 0 ? 0 : 1};
    
            \node[anchor=south west, font=\small, fill=white, inner sep=2pt]
                at (axis cs: 0.5, 1.1) {$g(x) = \begin{cases} 1 & \lfloor x \rfloor \text{ is odd} \\ 0 & \lfloor x \rfloor \text{ is even} \end{cases}$};
    
        \end{axis}
    \end{tikzpicture}
    \caption{$g(x)$, the square wave function.}
    \label{fig:parity_plot}
\end{figure}

\begin{proof}
    Since $\classG$ contains a single fixed function, it cannot shatter any points and has VC-dimension $0$.
    
    Let $w$ vary over the unit circle. Then for any $x \in \R^n$, we have $w \cdot x = \Norm{w}\Norm{x}\cos{\theta_w} = \Norm{x} \cos{\theta_w}$, where $\theta_w$ is the angle between $x$ and $w$.

    Consider the following $m$ points in $\R^n$: \[x_i = (2^i, 0, \dots, 0), \quad \forall i \in \{1, 2, \dots, m\}.\] These satisfy $\Norm{x_i} = 2^i$. For any given labeling of these points $(x_1, y_1), \ldots, (x_m, y_m)$ where $y_i \in \{0,1\}$, choose $w$ on the unit circle so that \[\theta_w = \arccos{\Paren{\frac{y_1}{2} + \frac{y_2}{2^2} + \ldots + \frac{y_m}{2^m}}}.\] 

    \begin{figure}[h]
        \centering
        \begin{tikzpicture}
            \begin{axis}[
                axis lines = middle,
                xmin = -1, xmax = 10,
                ymin = -1, ymax = 8,
                xlabel = {$x$},
                ylabel = {$y$},
                xtick = {2, 4, 8},
                ytick = \empty,
                width = 10cm, height = 6cm,
                axis line style={-stealth},
                clip = false
            ]
    
                \addplot [
                    only marks,
                    mark=*,
                    mark size=2.5pt,
                    blue
                ] coordinates {(2,0) (4,0) (8,0)};
    
                \coordinate (O) at (axis cs:0,0);
                \coordinate (X) at (axis cs:5,0); 
                \coordinate (W) at (axis cs:10, 6); 
    
                \draw[->, thick, red] (O) -- (W) node[above right] {$w$};
    
                \pic [draw, "$\theta_w$", angle radius=1.5cm, angle eccentricity=1.2] {angle = X--O--W};
    
            \end{axis}
        \end{tikzpicture}
        \caption{For $m = 3$ and $n = 2$, points $x_i$ along the $x$-axis and the weight vector $w$, where the angle $\theta_w$ is the angle between the $x$-axis (since all $x_i$ lie on it) and $w$.}
        \label{fig:vc_dimension_construction}
    \end{figure}
    
    We choose this $\theta_w$ so that $\cos{\theta_w}$ encodes the $y_i$ labels in binary, given by $(0.y_1y_2\ldots y_m)_2$. Multiplying by $\Norm{x_i} = 2^i$ effectively bit shifts so that the last digit in front of the decimal point is $y_i$. That is, for a given point $x_i$, $w \cdot x$ evaluates to 
      \begin{align*}
        \Norm{x_i} \cos{\theta_w} &= 2^i \cdot \Paren{\frac{y_1}{2} + \frac{y_2}{2^2} + \cdots + \frac{y_m}{2^m}} \\
        &= 2^{i-1} \cdot y_1 + \cdots + 2 \cdot y_{i-1} + y_i + 2^i \Paren{\frac{y_{i+1}}{2^{i+1}} + \cdots + \frac{y_m}{2^m}} \\
        &= (y_1y_2\cdots y_i.y_{i+1}\cdots y_m)_2
    \end{align*}

    The floor of this function is \[2^{i-1} \cdot y_1 + \ldots + 2 \cdot y_{i-1} + y_i = (y_1\cdots y_i)_2,\] whose parity only depends on $y_i$. Then by definition of $g$, we will have $g(w \cdot x_i) = y_i$. The number of points thus shattered can be arbitrarily large, so $\mathcal{F}$ has infinite VC-dimension.
\end{proof}

\section{\texorpdfstring{Proof of Theorem \ref{thm:VCD_semisubspace}}{Proof of Theorem 7}}

We derive the VC-dimension bound on semi-algebraic subspace juntas by observing that they are an instantiation of the class considered in Theorem 8.3 of \citep{anthony2009neural}. We first include the theorem statement here for completeness.

\begin{theorem}[\cite{anthony2009neural}, Theorem 8.3]\label{thm:VCtcombo}
    Let $H$ be a class of functions mapping from $\R^d \times \R^n$ to $\R$ so that for all $x \in \R^n, h \in H$, the function $a \mapsto h(a,x)$ is a polynomial on $\R^d$ of degree no more than $r$. 
    
    Suppose that $F$ is a $t$-combination of $\sign(H)$; that is, there exist a fixed boolean function $b: \{0,1\}^t \rightarrow \{0,1\}$ and fixed functions $h_1, \ldots, h_t$ in $H$ such that, for every $f \in F$, there is a parameter vector $a \in \R^d$ for which
    \[
        f(x) = b(\sign(h_1(a,x)), \ldots, \sign(h_t(a,x)))
    \]
    for all $x \in \R^n$. Then 
    \[\VCdim(F) \le 2d\log_2(12tr).\]
\end{theorem}

We now restate \Cref{thm:VCD_semisubspace} here for convenience.
\begin{theorem}[\Cref{thm:VCD_semisubspace}, restated]
    Let $\mathcal{F}^b_{n,\ell,k,t}$ be a class of semi-algebraic subspace juntas. Then, its VC-dimension can be upper bounded as follows:
    \[
        \VCdim(\mathcal{F}^b_{n,\ell,k,t}) \le 2d \log(12t(\ell+1)),
    \]
    where $d = kn + t{{k+\ell} \choose \ell}.$
\end{theorem}

\begin{proof}
    First, we parameterize the class. Any $f \in \mathcal{F}^b_{n,\ell,k,t}$ by definition can be written as
    \[
        f(x) = b(\sign(q_1(Wx)), \ldots, \sign(q_t(Wx))),
    \]
    where each $q_i$ is a polynomial of degree $\le \ell$ corresponding to the PTFs in $f$. It can be parameterized by a vector $a \in \R^d$, where $a$ contains the $kn$ entries of $W$ and the ${{k+\ell} \choose \ell}$ coefficients of each of the $t$ polynomials $q_1, \ldots, q_t$. This makes $d = kn + t{{k+\ell} \choose \ell}$ the number of parameters stored in $a$.
    
    More specifically, define $H := \{h_1, \ldots, h_t\}$, where $h_i(a,x) := q_i(Wx)$ and $W$ and the coefficients of $q_i$ are specified by $a$. For each fixed $x$, the function $a \mapsto h_i(a,x)$ is a polynomial of degree $\le \ell + 1$, since $q_i$ has degree $\le \ell$, $Wx$ is linear in the entries of $W$, and the coefficients of $q_i$ also appear in $a$.

    Thus $\mathcal{F}^b_{n,\ell,k,t}$ is a $t$-combination of $\sign(H)$ with the fixed boolean function $b$. Applying \Cref{thm:VCtcombo} yields the stated bound.
\end{proof}

\section{Experiments}\label{app:experiments}

\subsection{Overview}
To investigate the practical relevance of sparsity-induced generalization, we examine the $L_1$-induced sparsity on the ACSIncome prediction task using the 2024 Folktables dataset \citep{ding2022retiringadultnewdatasets} (sourced from the US Census). We evaluate OOD generalization of this task (predicting whether a person's income is $>\$50$K) by training on data from California and testing on a geographic shift to Texas. After filtering for the top 30 most common occupations to control categorical cardinality, the dataset consists of $81,451$ CA samples and $58,797$ TX samples. Using the union of features from all Folktables tasks and one-hot encoding categorical data, we construct a high-dimensional pool of $711$ columns. Of these, only $287$ of these columns correspond to features from the ACSIncome feature set; the rest serve as ``nuisance" features which are not necessarily relevant to predicting income.

\subsection{Experimental Setup}

\paragraph{Dataset.}
We use 2024 American Community Survey (ACS) data through the Folktables interface \citep{ding2022retiringadultnewdatasets}, which processes data from the US Census and contains several pre-determined prediction tasks. The 2024 calendar year is the most recent year with full data available. We take California to be the training environment and Texas as the out of distribution test environment.

\paragraph{Task.}
One of the predetermined tasks in Folktables is ACSIncome, which is to predict whether annual person income was greater than \$50,000 in a year. We apply Folktables' filter for adults with valid income, as ACSIncome does already. Since we use $L_1$ regularized logistic regression, we needed to one hot encode the categorical data, which was done on each state independently. 

The ten features that Folktables associated with ACSIncome were: age (AGEP), class of worker - private/public/self/etc. (COW), educational attainment (SCHL), marital status (MAR), occupation code (OCCP), place of birth (POBP), relationship to householder (RELSHIPP), usual hours worked per week (WKHP), sex (SEX), and race (RAC1P). As OCCP consisted of 530 different categories, in order to limit this cardinality among the expanded one hot encoded feature set, we restrict both states to rows whose OCCP is among the top 30 codes by frequency in CA. Thus OCCP's one hot encoded expansion is capped at 30 columns. After the filtering, there were $81,451$ entries ($48.2\%$ positive) for CA and $58797$ for TX ($44.4\%$ positive).

\paragraph{Feature Pool.} We took the 29-feature union of the five canonical Folktable benchmark tasks (ACSIncome, ACSEmployment, ACSPublicCoverage, ACSMobility, ACSTravelTime), excluding PINCP itself, which was the target (individual annual income).

\paragraph{Preprocessing.} Missing values were filled with $-1$. Categorical features were one-hot encoded on each state independently. They were then left-joined onto CA's column index: TX-only categories are dropped, and CA-only categories were zero-filled in the TX matrix. Columns were scaled by their CA training std to preserve the zero-structure of the one-hot indicators); zero-variance columns were left unscaled. The full design matrix had $711$ columns.

\paragraph{Model and Regularization Values.} The model was $L_1$ regularized logistic regression (scikit-learn, liblinear solver, 500 maximum iterations, tolerance $10^{-3}$). For each of five random seeds, we drew a stratified 80/20 train/validation split of CA, fit the model on CA train at each $C$ in a 17-point grid spanning $C \in [2 \times 10^{-4}, 10^2]$ (densest in the active region $C \in [2\times 10^{-4},\ 8\times 10^{-3}]$, sparser in the saturation tail), and evaluate on CA validation and the full filtered TX as the OOD test. We did not subsample and used the full dataset. The training and testing were performed on a single laptop and required approximately 30 minutes to run; not much compute was needed.

\paragraph{Analysis.} We plotted the balanced accuracy on the train/validation/OOD splits to avoid bias due to label against the number of nonzero coefficients. The features considered in the active set were those which had nonzero coefficients. We additionally computed, for each column, the total-variation distance between the train and OOD marginals, and plotted the maximum TV over the active feature set. Each plot showed the mean over the five seeds and plotted error bars with $\pm 1$ standard deviation for each value of $C$. There were $\pm 1$ standard deviation vertical error bars in the accuracy and active-set TV plots as well.

\subsection{Results}

We use $L_1$ regularized logistic regression with a sweep of various regularization values to illustrate the effect of sparsity on generalization. As Figure \ref{fig:acc_CA_to_TX} shows, both the training and validation accuracy (on CA) monotonically increase with the number of features. On the other hand, OOD accuracy peaks at higher sparsity ($101$ features). As the model becomes denser, OOD performance decreases alongside an increase in the maximum TV distance among active features around 100 nonzero features.

In this experiment, sparser models use features with lower individual TV distances. These distances provide an illustrative diagnostic of feature shift, rather than a verification of the joint-marginal overlap condition in \Cref{orig_feat_theorem,thm:orig_sparse_thm}. That condition concerns the joint distribution on $A = \Feat(f) \cup \Feat(h)$, whereas we measure individual marginals only and do not observe the target function's feature set. Thus, the experiment illustrates an association between sparsity and OOD performance.

\begin{figure}[H]
    \centering
    \begin{subfigure}[t]{\linewidth}
        \centering
        \includegraphics[width=\linewidth]{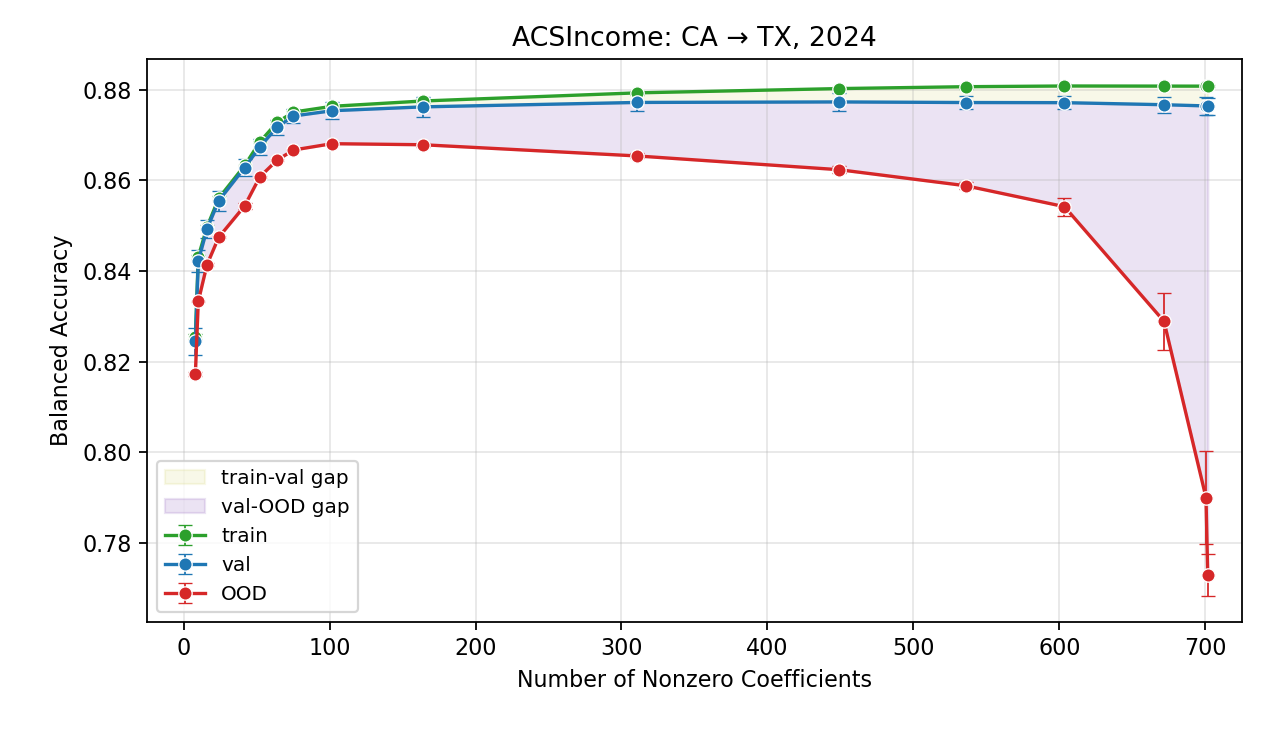}
        \caption{Balanced accuracy on train, validation, and OOD.}
        \label{fig:acc_CA_to_TX_acc}
    \end{subfigure}
    \hfill
    \begin{subfigure}[t]{\linewidth}
        \centering
        \includegraphics[width=\linewidth]{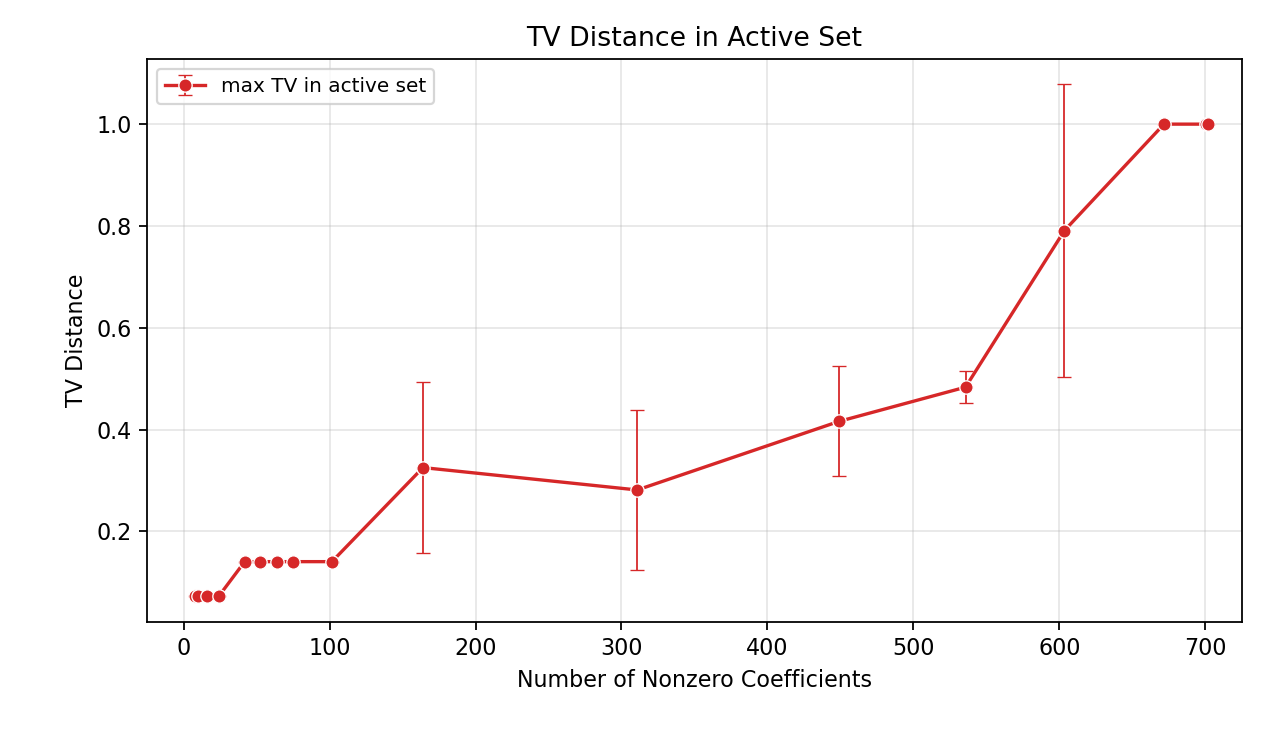}
        \caption{Max total variation distance for active features.}
        \label{fig:acc_CA_to_TX_tv}
    \end{subfigure}
    \caption{$L_1$ sweep for ACSIncome from CA to TX. \textbf{(a)} OOD accuracy peaks around $101$ nonzero features and declines as the model becomes too dense. \textbf{(b)} As number of nonzero features increases, $L_1$ starts including features with higher CA $\rightarrow$ TX TV distance. This climb in max active TV distance around $100$ nonzero features corresponds to the OOD-accuracy drop in the previous plot.}
    \label{fig:acc_CA_to_TX}
\end{figure}

The following is the full table of the values of $C$, the corresponding number of nonzero features, and the accuracies.

\begin{table}[h]
\centering
\small
\begin{tabular}{rrccc}
\toprule
$C$ & $L_0$ & Train & Val & OOD \\
\midrule
0.0002    & $8.0 \pm 0.0$    & $0.8256 \pm 0.0006$ & $0.8245 \pm 0.0030$ & $0.8172 \pm 0.0004$ \\
0.0003013 & $10.0 \pm 0.0$   & $0.8432 \pm 0.0005$ & $0.8422 \pm 0.0025$ & $0.8333 \pm 0.0001$ \\
0.000454  & $16.0 \pm 0.7$   & $0.8498 \pm 0.0004$ & $0.8493 \pm 0.0020$ & $0.8414 \pm 0.0003$ \\
0.000684  & $24.0 \pm 2.0$   & $0.8561 \pm 0.0005$ & $0.8555 \pm 0.0022$ & $0.8475 \pm 0.0002$ \\
0.001031  & $41.8 \pm 0.4$   & $0.8633 \pm 0.0002$ & $0.8629 \pm 0.0019$ & $0.8544 \pm 0.0006$ \\
0.001553  & $52.2 \pm 0.8$   & $0.8685 \pm 0.0003$ & $0.8674 \pm 0.0017$ & $0.8608 \pm 0.0002$ \\
0.002339  & $64.0 \pm 2.0$   & $0.8728 \pm 0.0004$ & $0.8718 \pm 0.0017$ & $0.8645 \pm 0.0002$ \\
0.003524  & $75.2 \pm 1.3$   & $0.8751 \pm 0.0003$ & $0.8742 \pm 0.0016$ & $0.8667 \pm 0.0001$ \\
\textbf{0.00531} & $\mathbf{101.6 \pm 1.1}$ & $\mathbf{0.8764 \pm 0.0002}$ & $\mathbf{0.8754 \pm 0.0019}$ & $\mathbf{0.8681 \pm 0.0002}$ \\
0.008     & $164.0 \pm 5.1$  & $0.8775 \pm 0.0004$ & $0.8762 \pm 0.0023$ & $0.8679 \pm 0.0004$ \\
0.015     & $310.8 \pm 5.7$  & $0.8793 \pm 0.0004$ & $0.8772 \pm 0.0018$ & $0.8654 \pm 0.0005$ \\
0.02823   & $449.2 \pm 7.6$  & $0.8803 \pm 0.0004$ & $0.8773 \pm 0.0019$ & $0.8624 \pm 0.0008$ \\
0.05313   & $535.8 \pm 7.9$  & $0.8807 \pm 0.0004$ & $0.8772 \pm 0.0015$ & $0.8588 \pm 0.0007$ \\
0.1       & $603.6 \pm 13.2$ & $0.8808 \pm 0.0003$ & $0.8772 \pm 0.0014$ & $0.8542 \pm 0.0020$ \\
0.5       & $671.8 \pm 1.6$  & $0.8808 \pm 0.0002$ & $0.8767 \pm 0.0017$ & $0.8289 \pm 0.0063$ \\
5         & $700.6 \pm 3.2$  & $0.8808 \pm 0.0003$ & $0.8765 \pm 0.0019$ & $0.7900 \pm 0.0102$ \\
100       & $701.8 \pm 3.0$  & $0.8808 \pm 0.0003$ & $0.8763 \pm 0.0020$ & $0.7729 \pm 0.0046$ \\
\bottomrule
\end{tabular}
\caption{$L_1$ regularization sweep for ACSIncome (CA $\to$ TX). Reported as mean $\pm$ standard deviation across seeds.}
\label{tab:l1_sweep}
\end{table}

\end{document}